\documentclass[letterpaper,11pt,leqno]{article}
\usepackage{paper}
\hypersetup{pdftitle={MarketBench: Evaluating AI Agents as Market Participants}}

\newcommand{\bib}{paper.bib}

\begin{document}

\title{MarketBench: Evaluating AI Agents as Market Participants}

\author{Andrey Fradkin, Rohit Krishnan
%
\thanks{Andrey Fradkin: Boston University and the MIT Initiative on the Digital Economy. Rohit Krishnan: Independent Researcher.}}

\date{\today}   


\begin{titlepage}
\maketitle
Markets are a promising way to coordinate AI agent activity for similar reasons to those used to justify markets more broadly. In order to effectively participate in markets, agents need to have informative signals of their own ability to successfully complete a task and the cost of doing so. We propose MarketBench, a benchmark for assessing whether AI agents have these capabilities. We use a 93-task subset of SWE-bench Lite, a software engineering benchmark, with six recently released LLMs as a demonstration. These LLMs are miscalibrated on both success probability and token usage, and auctions built from these self-reports diverge from a full-information allocation. A follow-up intervention where we add information about capabilities from prior experiments to the context improves calibration, but only modestly narrows the gap to a full-information benchmark. We also document the performance of a market-based scaffolding with these LLMs. Our results point to self-assessment as a key bottleneck for market-style coordination of AI agents.

\end{titlepage}

\section{Introduction}\label{s:introduction}
 
Agent benchmarks mostly ask whether a model can finish a task. Under market-style rules, agents typically name a price and convey how likely they are to succeed \emph{before} a task is assigned, so allocation turns on self-reported odds and expected cost, not only on whether a run ultimately succeeds. We propose MarketBench, a benchmark for this self-assessment on realistic software-engineering tasks. Our initial instantiation uses 93 SWE-bench Lite tasks and six recently released LLMs. Models have weak self-assessment: current models are poor at forecasting both their success probability and their token usage, and those errors carry over into auction outcomes built from the same self-reports. 



Why focus on this capability? Multi-agent systems are quickly becoming the standard way to do complex work through tools such as Claude Code and OpenClaw. However, these systems are complex and have ad-hoc rules for determining when and which types of agents are used across situations. In contrast, markets simplify complexity through the price mechanism. They are attractive for coordinating heterogeneous agents when task-specific fit, cost, or speed are local to the worker rather than visible to a central planner. But this logic only works if workers can say something informative about their own likely performance. Without calibrated self-assessment, bids are noisy from the start and downstream allocation suffers.\footnote{This difficulty is compounded by what \citet{dellacqua2023jagged} call the ``jagged technological frontier'': AI capability boundaries are irregular and unpredictable, with models excelling at some tasks while failing at others of seemingly similar difficulty.}

One might object that there is little genuinely private information across models: frontier systems are few in number and the operator can in principle observe them directly. That objection has two responses. Narrowly, heterogeneity in realized performance and cost reflects high-dimensional, task-specific model behavior; even if the architecture is public, the implied mapping from a task to success probability and token use is not fully predictable ex ante. More broadly, agents are bundles of a base model, execution environment, scaffolding, and operator-provided state; when the operator differs from the end user or principal, information about fit and cost is genuinely decentralized relative to a central market maker.

MarketBench is based on SWE-bench Lite software-engineering tasks, and we intend to extend it to other tasks. In its current form, MarketBench consists of two task families:
\begin{enumerate}
	\item MarketBench Calibration: We ask agents to forecast whether they can complete a task in one attempt and the total token usage they expect that attempt to consume. We then compare those reports to realized model runs and map the token forecasts into expected dollar costs using model-specific pricing.
	\item MarketBench Auction: We use the previously elicited success and token estimates to derive bids for a reserve-price procurement simulation and measure the resulting payoffs relative to full information about capabilities and costs.
\end{enumerate}

Separately, we present an illustrative scaffold experiment that elicits beliefs to route work across agents. We find that the scaffold benefits from using a diversity of agents, but is limited by the weak self-assessment of these agents. We treat this scaffold as a prototype of a system that can be improved in the future when agents improve their self-assessment capabilities.

The rest of this paper proceeds as follows. Section~\ref{sec:related} reviews related work on multi-agent coordination, LLM self-evaluation, and market mechanisms for computational resource allocation. Section~\ref{sec:model} presents a simple model of market-based task allocation among heterogeneous AI agents and derives the conditions under which market coordination outperforms fixed non-market routing rules. Section~\ref{sec:benchmark} introduces MarketBench, describing the task sets, bidding protocols, evaluation metrics, and benchmark results. Section~\ref{sec:scaffold} then presents our illustrative live scaffold experiment together with its empirical results. Section~\ref{sec:discussion} discusses the implications for AI agent design, hybrid market--AI coordination, and the path toward market-ready agents.

\section{Related Work}\label{sec:related}

A growing literature proposes multi-agent LLM systems that decompose work across role-specialized agents and coordinate through explicit prompting, message passing, and hand-designed workflows \citep{wu2023autogen,hong2023metagpt,li2023camel}. Commercial coding agents similarly rely on centralized scaffolds that route subtasks, manage tool use, and implement retry logic \citep{anthropic2025claudecode,openai2025codex}. More dynamic orchestrators replace static chains with planner--scheduler architectures, but the core paradigm remains hierarchical: a central controller must decide decomposition, assignment, and failure recovery \citep{song2025gradientsys,tomasev2026intelligent}. Related work also explores market-inspired routing rules. For example, \citet{alazraki2026sale} uses ``strategy auctions'' to score agent plans and route work among heterogeneous models, but this is not a literal market mechanism in our sense because bids do not arise from agents pricing their own expected success and cost under an incentive-compatible allocation rule. Our focus differs in asking whether coordination can be decentralized through market mechanisms, taking inspiration from how markets have been used at scale in society, and what agent capabilities such mechanisms presuppose.

The idea of bid-based allocation in multi-agent systems predates LLMs. Classical distributed AI introduced contracting and bidding protocols for task allocation, and the computer-systems literature developed market-based approaches to allocate fungible computational resources. Market-oriented programming, for example, computes allocations via competitive-equilibrium prices in artificial economies \citep{wellman1993market}. Those mechanisms work when agents can price their own costs and constraints accurately. In our setting, the traded object is heterogeneous cognitive labor whose quality and cost depend on model, context, and tools, so meaningful bids require metacognitive self-assessment.

A second, related, literature evaluates frontier agents on complex tasks. A popular example is METR's time-horizon methodology, which characterizes how long an agent can sustain autonomous work at a given reliability level \citep{kwa2025measuring}. GDPval measures performance on economically meaningful tasks mapped to real occupational work activities \citep{patwardhan2025gdpval}. The Holistic Agent Leaderboard (HAL) contributes evaluation infrastructure that standardizes comparisons across models, scaffolds, and benchmarks, while using large-scale rollout logs to surface behavioral failure modes that summary scores can miss \citep{kapoor2025hal}. Our closest reference is \citet{rabanser2026reliability}, which argues that single success metrics obscure important dimensions of agent performance and proposes a broader reliability profile spanning consistency, robustness, predictability, and safety. That paper also elicits the model's confidence that its answer is correct and evaluates those forecasts using the Brier score, as we do. Nonetheless, our elicitation target is different: we ask for ex ante beliefs about task-level success and expected cost, rather than confidence in the correctness of a realized answer. Together these benchmarks advance measurement beyond raw task success toward the operational properties that determine whether agents can be trusted in real deployments.

This connects naturally to work on LLM confidence elicitation and calibration. Numerous papers document miscalibration and elicitation sensitivity \citep{xiong2024uncertainty,geng2024survey,zhang2024fidelity}. We extend this agenda from answer-level confidence to agentic production: market participation requires calibrated beliefs about task completion and expected cost, beyond correctness alone. Finally, our motivation is grounded in classic arguments for decentralized coordination. Hayek emphasizes that prices aggregate dispersed local knowledge \citep{hayek1945}.

\section{Conceptual Framework}\label{sec:model}

In this section, we walk through a stylized model in which the value of a market mechanism for a principal looking to complete a task comes from task-specific private information. We consider this framework an intuition pump, rather than a full model of a market-based scaffolding.

A principal receives value $v>0$ if a task is completed. There are two agents, $H$ and $L$. Agent $H$ is the ``better'' agent on average but is also more expensive: it has higher baseline capability $a_H>a_L$ and higher cost $c_H>c_L>0$. Each task has difficulty $d$. For agent $i\in\{H,L\}$, realized task fit is $a_i+\varepsilon_i$, where $\varepsilon_i$ is an idiosyncratic draw. The agent succeeds if $a_i+\varepsilon_i\geq d$. If $F$ denotes the cumulative distribution function of $\varepsilon_i$, then the ex ante completion probability of agent $i$ on a task of difficulty $d$ is
\[
p_i(d)=1-F(d-a_i),
\]
so $p_H(d) \geq p_L(d)$ for every $d$.

Consider first simple non-market routing rules that ignore the task-specific draw. If the principal assigns the task only to the better agent, the expected value net of inference cost is
\[
W_H(d)=v p_H(d)-c_H,
\]
and the expected cost is
\[
C_H(d)=c_H.
\]
If instead the principal assigns the task only to the worse agent, the corresponding objects are
\[
W_L(d)=v p_L(d)-c_L,
\]
and
\[
C_L(d)=c_L.
\]
The better agent is preferred to the worse one when
\[
v\left[p_H(d)-p_L(d)\right]>c_H-c_L.
\]

If the principal runs both agents in parallel and takes a successful output from either one, then under independence of the two completion events the probability of success is
\[
Q_P(d)=1-\left(1-p_H(d)\right)\left(1-p_L(d)\right),
\]
so expected value and cost are
\[
W_P(d)=v Q_P(d)-c_H-c_L
\]
and
\[
C_P(d)=c_H+c_L.
\]
Relative to assigning only $H$, parallel execution is worthwhile when $v(1-p_H(d))p_L(d)>c_L$. Relative to assigning only $L$, it is worthwhile when $v(1-p_L(d))p_H(d)>c_H$. Thus, parallel execution trades off a higher completion probability against the waste created by paying both agents, especially when their success events are highly redundant.

Now suppose that before bidding each agent observes its own task-specific draw $\varepsilon_i$ and hence learns whether the task lies within its current capability frontier. To keep the algebra transparent, assume that after observing $\varepsilon_i$ the agent knows whether it can complete the task, so its realized completion probability is
\[
\pi_i(d,\varepsilon_i)=1\{a_i+\varepsilon_i\geq d\}.
\]
Consider a success-contingent procurement contract that pays $b_i$ only if the task is successfully completed. Agent $i$'s expected payoff from such a contract is $\pi_i(d,\varepsilon_i)b_i-c_i$, so a zero-profit bid is $b_i=c_i/\pi_i(d,\varepsilon_i)$, interpreted as $+\infty$ when $\pi_i(d,\varepsilon_i)=0$. The principal therefore awards the task to the cheapest agent among those who know they can solve it, provided the implied payment is below $v$.

Under independence, this yields a simple expression for expected value. Let $p_i(d)$ denote the ex ante probability that agent $i$ declares itself able to solve the task. Since the worse agent is cheaper, it wins whenever it is capable; the better agent is used only when the worse agent cannot solve the task and the better one can. Hence the market allocation has expected value
\[
W_M(d)=p_L(d)(v-c_L)+\left(1-p_L(d)\right)p_H(d)(v-c_H),
\]
expected cost
\[
C_M(d)=p_L(d)c_L+\left(1-p_L(d)\right)p_H(d)c_H,
\]
and completion probability
\[
Q_M(d)=p_L(d)+\left(1-p_L(d)\right)p_H(d).
\]
The market allocation is valuable because it conditions on private, task-specific information rather than on average success rates alone.

To compare the welfare of the different allocation rules more transparently, it is useful to work directly with ex post capability states. Let $\lambda_{\ell h}(d)=\Pr(\pi_L=\ell,\pi_H=h)$ for $\ell,h\in\{0,1\}$, where the first index refers to $L$ and the second to $H$. Thus $\lambda_{10}$ is the probability that only $L$ can solve the task, $\lambda_{01}$ is the probability that only $H$ can solve it, $\lambda_{11}$ is the probability that both can solve it, and $\lambda_{00}$ is the probability that neither can solve it. The expected welfare of the market mechanism can then be written as
\[
W_M(d)=\lambda_{10}(v-c_L)+\lambda_{01}(v-c_H)+\lambda_{11}(v-c_L),
\]
since the market picks the cheapest capable agent in each state and does not allocate the task when neither agent can solve it.

\begin{proposition}\label{prop:market-dominance}
Suppose $v>c_H>c_L>0$. Then the market allocation weakly dominates each non-market alternative. Moreover, the dominance is strict whenever there is positive probability of a state in which the alternative either allocates the task to the wrong agent, pays a redundant agent, or pays for an attempt when no agent can solve the task.
\end{proposition}

The intuition is straightforward. Relative to fixed assignment, the market preserves the good states while correcting the bad ones: it uses $L$ when only $L$ can solve, uses $H$ when only $H$ can solve, uses the cheaper agent when both can solve, and avoids wasting cost when neither can solve. Relative to parallel execution, the market achieves the same completion outcome in every state but avoids paying redundant costs. The proof is in Appendix~\ref{app:proof-market-dominance}.

The assumption that agents perfectly observe $\varepsilon_i$ is stronger than what practice requires. In realistic systems, agents only need a signal about likely success or cost that is more informative, or at least not fully redundant, relative to the principal's own signal. If an agent's self-assessment is just a noisy restatement of what the principal already knows, bidding adds little. But if self-assessment contains even modest incremental information about task-specific fit, then bidding can improve allocation and reduce unnecessary parallelism.

\section{MarketBench}\label{sec:benchmark}

MarketBench is a benchmark for AI agents' readiness for market participation. Our initial instantiation uses SWE-bench Lite tasks \citep{jimenez2024swebench}: real GitHub issue-fix pairs with executable test suites and a binary notion of success. Each SWE-bench Lite task is a single GitHub issue drawn from one of roughly a dozen mature open-source Python repositories (\texttt{django}, \texttt{sympy}, \texttt{scikit-learn}, \texttt{sphinx}, etc.); a given repository typically contributes multiple tasks, so per-repository performance priors are well-defined. The novelty of MarketBench lies in what we ask agents to reveal about those tasks before allocation. We study whether models can forecast their own success probabilities and expected resource use well enough to support decentralized routing.

\subsection{Calibration}

The calibration task asks a model to inspect a software-engineering task and forecast whether it can solve that task in one attempt, together with the total token budget it expects to consume. Before the solve attempt, the model is prompted as follows:

\begin{quote}
\small
Estimate the probability that you could complete this task correctly in one attempt. Return JSON only with fields \texttt{p\_success} (0..1), \texttt{estimated\_tokens\_total} (total model tokens for one full solve attempt), \texttt{rationale} (optional).
\end{quote}

We compare the reported \texttt{p\_success} to realized pass/fail outcomes and map token forecasts into expected dollar costs using model-specific blended pricing. Phase~I uses 93 SWE-bench Lite tasks and six frontier models, yielding 558 model-task rows. The elicitation prompt exposes the task ID, title, full problem statement, and acceptance commands, and responses are collected as strict JSON at temperature 0. Realized success labels come from strong-scaffold external SWE-bench runs. By ``external scaffold,'' we mean a stronger SWE-bench execution environment with interactive shell access, direct test execution and feedback, structured file editing, and multi-turn revision.

\subsection{Auction}

The second task family asks whether those self-assessments are good enough to support bidding. We study a simple reserve-price procurement simulation in which each model's bid is mechanically derived from its previously elicited success probability and token-cost estimate rather than from a separate bidding prompt. A given auction only includes one model as a participant, but is incentive-compatible given that it is a second-price auction with an independently drawn reserve.

The model-specific breakeven bid is
\[
b^*=\frac{\text{token cost}+\text{penalty}\times(1-p_{success})}{p_{success}}.
\]
The simulation uses model-specific blended token pricing,\footnotetext{Blended pricing is a single dollars-per-token rate for each model: we divide the run's total dollar cost by its total token count (input and output pooled), so all token categories are valued at that common average.} a failure penalty of \$2, and reserve prices drawn independently from a $\text{Uniform}[0,1]$ distribution, with 100 reserve draws per row and deterministic seeds. Expected profit uses the model's reported \texttt{p\_success}; realized profit evaluates the same mechanically derived bid against the observed pass indicator.

\subsection{Calibration Results}

We begin by evaluating whether models can predict their own task success on 93 SWE-bench Lite tasks. Table~\ref{t:calibration-results} reports mean stated success probability, realized pass rate, Brier skill relative to a base-rate benchmark, and two summary statistics for token forecasting.

\begin{table}[t]
\caption{Calibration on 93 SWE-bench Lite tasks}
\label{t:calibration-results}
\begin{tabular*}{\textwidth}{@{\extracolsep\fill}lccccc}
\toprule
Model & Mean $p$ & Pass rate & Skill & Est.\ toks & Realized toks \\
\midrule
Claude Opus 4.5 & 76.5\% & 80.6\% & +0.060 & 14,774 & 35,820 \\
Gemini 3 Pro Preview & 92.9\% & 80.6\% & -0.111 & 1,801 & 55,969 \\
GPT-5.2 & 63.2\% & 80.6\% & -0.195 & 2,909 & 25,036 \\
GPT-5.2-pro & 66.3\% & 79.6\% & -0.111 & 2,647 & 3,970 \\
Claude Sonnet 4.5 & 75.8\% & 77.4\% & +0.018 & 18,333 & 53,085 \\
GPT-5-mini & 61.4\% & 75.3\% & -0.305 & 2,595 & 28,276 \\
\bottomrule
\end{tabular*}
\note{The pass-rate spread is much narrower than the confidence spread. Only Claude Opus 4.5 and Claude Sonnet 4.5 achieve positive Brier skill relative to a naive base-rate forecast. The final two columns summarize the magnitude of token-estimation error: estimated tokens are the model's forecast, and realized tokens are the actual usage implied by dollar cost and model-specific pricing.}
\end{table}

The main fact from Phase~I is miscalibration. All six models cluster between 75.3\% and 80.6\% realized pass rates, yet their average stated success probabilities range from 61.4\% to 92.9\%. Gemini 3 Pro Preview is sharply overconfident, while GPT-5.2 and GPT-5-mini are underconfident relative to their realized pass rates. To support the illustrative live scaffold case study below, we select a common subset of 50 tasks that all six models attempted; on that slice, GPT-5.2 is the strongest single-model baseline on the external scaffold at 37/50 tasks, while a simple oracle that always picks the best model reaches 42/50. Token self-estimates are also severely understated: the published global median estimated-to-actual token ratio is 0.02. On the same 50-task slice, a later Codex follow-up routes the same GPT-5.2 model family through a different execution path, called \texttt{codex-direct}, with an 1800-second budget; we return to that diagnostic harness comparison below.

We also assessed whether prompt modification can partially address this bottleneck. For a given task, we computed performance on the other SWE-bench tasks and prepended it to the prompt. The performance metrics we included were pass rate, average stated confidence, and typical token underestimation; Appendix~\ref{app:self-knowledge-card} shows a representative example, and Table~\ref{t:self-knowledge-summary} reports the aggregate before-and-after calibration summary. On the full six-model rerun, the mean Brier score improves from 0.1835 to 0.1693 and ECE improves from 0.1065 to 0.0616; token forecasts also become less severely understated. Even a simple self-history card therefore makes the forecasts more useful. 

\subsection{Reserve-Price Auction Results}

We next ask whether these self-assessments, though noisy, are nevertheless sufficient for profitable bidding in the reserve-price auction described above. Table~\ref{t:reserve-auction-results} summarizes simulated outcomes on the same 93-task calibration set.

\begin{table}[t]
\caption{Single-model reserve-price auction outcomes}
\label{t:reserve-auction-results}
\begin{tabular*}{\textwidth}{@{\extracolsep\fill}lccccc}
\toprule
Model & $N$ & Win rate & Exp. profit & Realized profit & Oracle profit \\
\midrule
Claude Opus 4.5 & 93 & 28.2\% & \$0.077 & \$0.080 & \$0.260 \\
Claude Sonnet 4.5 & 93 & 29.3\% & \$0.062 & \$0.064 & \$0.288 \\
Gemini 3 Pro Preview & 93 & 84.6\% & \$0.353 & \$0.303 & \$0.391 \\
GPT-5-mini & 93 & 21.6\% & \$0.054 & \$0.050 & \$0.375 \\
GPT-5.2 & 93 & 5.4\% & \$0.005 & \$0.006 & \$0.385 \\
GPT-5.2-pro & 93 & 5.3\% & \$0.007 & \$0.007 & \$0.227 \\
\bottomrule
\end{tabular*}
\note{Profits are reported in dollars per task under the reserve-auction simulation. The oracle column assumes perfect knowledge of which tasks are actually solvable and abstention on the rest.}
\end{table}

Two patterns stand out. First, all models earn less than the oracle benchmark, often by a wide margin. GPT-5.2, for example, earns roughly \$0.006 per task in realized profit versus \$0.385 for its oracle counterpart. Second, poor calibration distorts allocation even when average profits remain positive. Gemini wins 84.6\% of simulated auctions and earns the largest realized profit, but this dominance is driven by aggressive bidding despite the weak calibration signal documented above. The resulting auction is therefore far from the full-information outcome. 

The self-knowledge follow-up from the preceding subsection partly explains why. It improves the underlying forecasts and shifts allocation in the predicted direction, but the aggregate auction numbers move only at the margin (mean realized profit \$0.085 to \$0.082, oracle gap 0.24 to 0.23). Prompt-level self-knowledge is enough to repair \emph{who} wins; it is not yet enough to repair \emph{how much} the principal recovers. This section does not report a full reserve-distribution or payment-rule sensitivity sweep. Later informed competitive auctions suggest that prompt framing changed bidding behavior more than payment-rule variation changed allocation accuracy. The deeper bottleneck still appears to be weak task-by-task self-assessment.

\section{Illustrative Scaffold Experiment}\label{sec:scaffold}

\subsection{Setup}

Beyond the benchmark, we built a market-inspired scaffold to test whether these model submitted bids improve allocation in a multi-agent setting. We call it the \emph{live scaffold}, contrasted with the \emph{external scaffold} used for benchmark labels and single-model baselines. The live scaffold treats each SWE-bench issue as a single task and verifies a submitted patch at the end. Each attempt is one shot --- a single patch, with no interactive shell, no test-feedback loop, and no within-attempt revision --- but each task gets up to two attempts. On a failed first attempt, a hard worker-exclusion rule forces the second attempt to a different model. The live scaffold is therefore a case study in routing under imperfect self-assessment, not a literal implementation of the success-contingent procurement contract analyzed in Section~\ref{sec:model}.

Our scaffold is operator-run rather than a clean procurement auction with outside suppliers. The operator chooses assignments, pays model-run costs directly, and uses a score as an internal routing heuristic. In this way, it is not a real market but market-like. That said, we also envision future variations where third-party agents bid in a decentralized manner and their owners bear the token costs incurred while performing the task. 

The scaffolding works as follows. Each available worker submits an ask, a self-assessed probability of success, and an expected completion time, although the current implementation does not yet use the time field in the routing score (it is used only as a tie-breaker). We model the worker as expecting to receive its ask minus its cost upon winning and successfully executing the task, and paying a penalty if it fails.

The operator pays execution costs and decides whether to assign a task and to whom based on a score. For bidder $i$,
\[
\text{Score}_i = p_i \times (\text{Utility} - \text{Ask}_i) - (1-p_i) \times \text{Penalty}_i - E[\text{Cost}_i].
\]
where $\text{Utility}$ is the operator's per-task utility (a scaffold-level constant), $\text{Ask}_i$ is bidder $i$'s requested payment on success, and $\text{Penalty}_i$ is the scaffold's explicit loss when bidder $i$'s attempt fails.\footnote{In the implementation, the failure penalty is itself a function of the bidder's stated $p_i$: $\text{Penalty}_i = \rho \cdot \text{Utility} \cdot (0.5 + p_i)$, with $\rho$ a fixed scaffold parameter. Overconfident bidders therefore face a heavier expected loss on failure.}

All bids with positive scores remain eligible, and the highest-scoring available worker receives the task. If that worker fails, the engine considers the next-highest eligible bid. Tasks are run with the same six frontier models used in the benchmark, a utility of 90 internal credits for every task, an explicit failure-penalty settlement rule, a 90-second bidding window, a 900-second execution limit, and at most two attempts per task. Workers are required to submit bids on tasks, and a worker that fails a task is excluded from retrying that same task.

\subsection{Scaffold Performance and Extensions}

We evaluate the live scaffold on the common 50-task subset introduced above. Table~\ref{t:market-vs-solo-results} reports the published market-versus-solo comparison, which is best read as a scaffold-level comparison rather than as a clean mechanism test. ``Market (our scaffold)'' is the six-worker live-routing condition described above. ``Solo GPT-5.2 (our scaffold)'' uses the same internal scaffold but always assigns GPT-5.2, including on retry. ``Solo GPT-5.2 (external scaffold)'' is the stronger SWE-bench baseline from Section~\ref{sec:benchmark}, and the oracle ceiling asks what pass rate an omniscient router could reach on the external scaffold by always picking the right model for each task.

\begin{table}[t]
\caption{Illustrative live scaffold experiment results on 50 common tasks}
\label{t:market-vs-solo-results}
\begin{tabular*}{\textwidth}{@{\extracolsep\fill}lcc}
\toprule
Execution paradigm & Pass rate & Passes \\
\midrule
Market (our scaffold) & 58.0\% & 29/50 \\
Solo GPT-5.2 (our scaffold) & 48.0\% & 24/50 \\
Solo GPT-5.2 (external scaffold) & 74.0\% & 37/50 \\
Oracle ceiling (external scaffold) & 84.0\% & 42/50 \\
\bottomrule
\end{tabular*}
\note{The same-scaffold comparison keeps the task slice, verifier, timeout, and two-attempt cap fixed, but not the worker pool: the market condition uses six workers and diverse-model retry, whereas the solo condition always uses GPT-5.2.}
\end{table}

The market-versus-solo comparison keeps the task slice and run policy fixed, but it does not keep the worker pool fixed. Specifically, both conditions use the same 50 tasks, verifier, 900-second execution limit, and cap of two attempts per task. But the market condition has access to a six-model worker pool and can route a retry to a different model, whereas the solo condition always uses GPT-5.2. Under those conditions, the market outperforms the same-scaffold solo baseline by 5 tasks, or 10 percentage points. 18 tasks are solved by both systems, 11 by the market only, 6 by solo GPT-5.2 only, and 15 by neither. That pattern is consistent with gains from model diversity and retry, but the p-value for the difference in means is $0.3$, so we cannot reject the hypothesis that the two systems are equivalent on this slice. We therefore treat this comparison as suggestive --- a diverse worker pool plausibly helps on our scaffold --- and rest the paper's main claims on the calibration evidence in Section~\ref{sec:benchmark} rather than on the scaffold-level pass-rate gap.

Table~\ref{t:market-vs-solo-results} also demonstrates the underperformance of the live scaffold relative to external scaffolds: the same GPT-5.2 loses 26 percentage points (13 tasks, 74\% vs.\ 48\%) when moved from the external scaffold to ours. The reason for this difference is that the external scaffold gives the model an interactive shell, direct test execution and feedback, multi-turn iteration, and structured file-editing tools; the live scaffold offers none of these and limits each worker to a single patch (raw diff or \texttt{BEGIN\_FILE} block) per attempt. The market recovers about 10 percentage points of that gap through model diversity; the rest would require scaffold upgrades, not better bidding.

The market run also used more compute than solo GPT-5.2, because multiple workers inspect each task before clearing and some tasks take a second attempt: about 5.82M total tokens for the market run versus 4.37M for the solo baseline (200.8K versus 182.2K tokens per successful pass).

\begin{table}[t]
\caption{Alternative Scaffold Specifications}
\label{t:scaffold-followups}
\begin{tabular*}{\textwidth}{@{\extracolsep\fill}lcc}
\toprule
Follow-up & Pass rate & Passes \\
\midrule
Matched centralized router & 54.0\% & 27/50 \\
Matched market rerun & 46.0\% & 23/50 \\
Market with self-knowledge bidding prior & 56.0\% & 28/50 \\
Codex + GPT-5.2 (1800s budget) & 70.0\% & 35/50 \\
\bottomrule
\end{tabular*}
\note{All rows use the common 50-task slice. The matched centralized-router and matched market rerun keep the same six workers, verifier, and 900-second execution limit. The self-knowledge market row changes only the private bidding context by adding a held-out calibration prior and bidding guardrails before bidding. The Codex row changes both execution path and budget, so it is a diagnostic rather than an apples-to-apples replacement for the published benchmark.}
\end{table}

There are two explanations for why our market scaffold can beat a solo model: the market-clearing rule, and the diverse worker pool.

To separate these explanations, we ran a follow-up that keeps everything identical, but replaces the market-clearing rule with a centralized router. The router is a single LLM call (GPT-5.2-pro) that reads the task descriptions, the available workers with their model identifiers, expected cost hints, retry-exclusion constraints, and other available context such as discussion notes and worker reputation summaries, and then directly assigns a worker to each task. The assigned worker still executes the task through the same path as in the market condition; the only difference is who chooses the assignment. Under those matched conditions, the centralized router reaches 27/50 while the market reaches 23/50. The gain from the published market comparison therefore appears to come primarily from model diversity rather than from the present bidding rule. This shouldn't be surprising, since models perform poorly in self-assessment as measured by MarketBench.\footnotetext{The main driver of the market's drop in the matched rerun is Gemini. Gemini bids aggressively, so the market selects it more often. But its realized performance is weak. This is a bid-calibration problem: an overconfident worker crowds out better-suited alternatives.}

It is worth flagging directly that the ``Market (our scaffold)'' row in Table~\ref{t:market-vs-solo-results} (29/50) and the ``Matched market rerun'' row in Table~\ref{t:scaffold-followups} (23/50) are nominally the same condition, run twice. The 6-task gap between them is itself informative: at this sample size, run-to-run variation from stochastic bids, retry ordering, and execution-path noise is on the same order as the differences between scaffolds we are trying to measure. Comparisons across rows of Table~\ref{t:scaffold-followups} should be read with that noise floor in mind, and any single 50-task scaffold evaluation --- ours included --- should be treated as imprecise on its own.

If overconfident bidding is the bottleneck, can better priors help? To test this, we ran a second follow-up that keeps the market-clearing rule but changes only the private bidding context: each worker receives a held-out calibration prior before bidding, summarizing its own historical pass rate and overconfidence tendency. The hard-prior rule (for the related calibration card see Appendix~\ref{app:self-knowledge-card}) starts from the held-out pass rate minus any historical overconfidence gap, then asks the worker to raise \texttt{p\_success} only when the current task gives direct evidence that it is easier than the held-out average. Under that intervention, the market reaches 28/50---up from 23/50 in the matched rerun and slightly above the centralized router at 27/50. This is suggestive that better self-knowledge would yield even greater performance benefits.

\paragraph{Sensitivity to execution harness and time budget.} A separate question is how much the live scaffold's weaker execution environment constrains all of these results. To probe this, we ran a diagnostic on the same 50-task slice using GPT-5.2 through a different execution route, called the \texttt{codex-direct} worker path, with the per-task limit doubled from 900 to 1800 seconds. Under those changed conditions, the follow-up reaches 35/50 passes, compared to 23--29/50 for our market scaffold and 24/50 for solo GPT-5.2 in our scaffold. The cost picture is very different, however. The \texttt{codex-direct} run consumes 321.3 million total tokens, versus 4.37 million for the solo GPT-5.2 run and 5.82 million for the market run --- roughly 55--75$\times$ more compute for, at most, a 12-task improvement over our market scaffold and an 11-task improvement over solo GPT-5.2. 

Two observations follow. First, when comparing models or scaffolds, the execution path is first order: routing and bidding choices that move pass rates by a few tasks are dwarfed by execution-environment effects. Second, the steep token cost of those gains is itself a reason to care about market-style coordination: if substantial pass-rate improvements available at the frontier come with order-of-magnitude compute increases, then mechanisms that route only the tasks that need that budget to the expensive path become economically important rather than cosmetic.

\section{Discussion}\label{sec:discussion}

As the number and variety of AI agents being used in production increases exponentially, identifying ways for them to efficiently and successfully collaborate on hard problems becomes a key bottleneck. For human agents, markets emerged as the dominant mechanism for this kind of coordination, and there are reasons to expect a similar role for AI agents.

Yet current agents do not know enough about their own production functions. They can solve complex tasks, but they cannot reliably say in advance which tasks they can solve, how likely they are to succeed, and what the attempt will cost. But that information is exactly what a market is supposed to aggregate. If bids do not reflect informed private information, the price system cannot do its job. Even though a brief summary of a model's own past behavior materially improves calibration, it is still far from the theoretical maximum.

These observations suggest that the optimal solution for coordinating agents will likely be a combination of markets and AI. Markets are useful when knowledge is local, dispersed, and hard to centralize. That description fits agentic work rather well. Different models have different strengths, different tool use, different failure modes, and different costs. But with imperfect self-knowledge and behavior, agents may not be reliable in their own assessments of their abilities. The role of AI in such a system is to augment the market, scoring agent bids based on more centralized sources of information, especially about the trustworthiness of an agent's bids.

An analogy to other types of auctions is useful. In procurement, buyers often care about more than price. They care about quality, speed, reliability, and other attributes of performance. Scoring auctions handle that problem directly: the winner is the bidder offering the best combination of price and non-price characteristics \citep{che1993design}. Similarly, ad auctions combine this idea with a statistical model of relevance \citep{varian2007position}.

A mature market for AI agents is likely to work similarly. An agent that is cheap but unreliable should have a low quality score. In contrast, an agent that is more expensive but bids accurately may be the efficient choice. When multiple agents bid, a scoring function should appropriately weigh the bids to determine the winner. Concretely, the quality score in such a market could combine an agent's held-out calibration record on a benchmark like MarketBench (how well its stated $p_{\text{success}}$ has tracked realized outcomes), domain-specific reputation aggregated across past tasks of similar type, and, where verification is possible, the rate at which a third-party verifier accepted its outputs. These are exactly the inputs that a centralized router currently uses ad hoc; making them explicit and tradeable across operators is what would turn the existing scaffolds into a market.

Recent evaluations also show that measured agent performance can keep improving at much larger inference budgets than standard setups typically allow \citep{aisi2026inference-scaling-cyber}, and our own \texttt{codex-direct} diagnostic in Section~\ref{sec:scaffold} points the same way: doubling the per-task budget and changing the execution path adds substantial passes at the price of roughly two orders of magnitude more compute. If pass-rate gains at the frontier come from spending more compute on a per-task basis, then a single scalar bid that quotes one expected cost is the wrong object. The right object is closer to a production schedule: a mapping from token budget to expected success probability, expected wall-clock time, and a sketch of how the agent would allocate that compute across search, tool use, and revision. A market that elicits such schedules can in principle let the principal buy the right point on each agent's cost-quality frontier, instead of being locked into whichever budget the agent happened to assume when it bid. We do not develop this here, but we view it as the natural next step once self-assessment of scalar success is itself reliable enough to build on.

The agenda that we propose has the following next steps. First, self-assessment should become a target of training and evaluation in its own right. Calibration, abstention, and budget discipline in bidding are core capabilities for decentralized coordination. Second, richer market institutions will need to arise, including reputation, escrow, and other mechanisms. Third, this benchmark should be extended beyond software engineering into other complex settings. 

\bibliography{\bib}

\appendix
\setcounter{table}{0}
\renewcommand{\thetable}{A\arabic{table}}

\section{Self-Knowledge Card for Direct Calibration}\label{app:self-knowledge-card}

The self-knowledge calibration follow-up prepended a short held-out history card before the standard direct calibration prompt. The exact numbers varied by model and by the held-out tasks available for that model.

SWE-bench instances are drawn from real open-source repositories, not synthetic toy problems. Task identifiers follow the usual convention: for example, \texttt{sympy\_\_sympy-17630} denotes an issue from the \texttt{sympy/sympy} repository. The main experiments use a fixed 93-task subset of MarketBench; that subset contains multiple tasks from some repositories, rather than a single task per repository. When we form the self-knowledge card for a given evaluation task, the held-out statistics exclude the current task. The figure ``across 92 held-out tasks'' aggregates all other tasks in this setup for the same model, spanning every repository represented in the subset. The line ``on prior tasks from \texttt{sympy/sympy} (21 held-out tasks)'' is the same-repository slice of that pool: among those 92 tasks, the ones that also come from \texttt{sympy/sympy}. The repository-specific line was included only when there was enough held-out data for that repository.

The template for the card was:

\begin{quote}
\small
Use the historical self-knowledge summary below as a prior for this model.

These statistics come from held-out MarketBench tasks for this same model and exclude the current task.

Historical self-knowledge summary:
\begin{itemize}[leftmargin=1.5em]
\item Across \texttt{\{N\_held\}}\ held-out tasks, your pass rate was \texttt{\{pass\_rate\}}.
\item Your mean previously stated success probability was \texttt{\{mean\_p\}}; you were historically \texttt{\{over\}} or \texttt{\{under\}}confident by \texttt{\{gap\}}.
\item Your actual solve-token usage was typically \texttt{\{ratio\}}x your estimate.
\item \emph{(Included when $\geq$5 held-out tasks from the same repository were available.)} On prior tasks from \texttt{\{repo\}} (\texttt{\{N\_repo\}}\ held-out tasks), your pass rate was \texttt{\{repo\_pass\_rate\}} and your mean stated success probability was \texttt{\{repo\_mean\_p\}}.
\end{itemize}

Start from these historical tendencies, then update using the specific evidence in the current task.

Be willing to move away from the historical average when the task details clearly support it.
\end{quote}

At runtime, placeholders were filled per model using leave-one-out statistics over the 93-task calibration set: \texttt{\{N\_held\}} was typically 92 (93 minus the current task, adjusted for any missing rows for that model). Because SWE-bench Lite tasks cluster by source repository, the optional per-repository line gave the model a more task-relevant prior when enough same-repo data existed.

As a concrete example, the card saved for Claude Opus~4.5 on the \texttt{sympy\_\_sympy-17630} instance had \texttt{\{N\_held\}}=92, \texttt{\{pass\_rate\}}=81.5\%, \texttt{\{mean\_p\}}=76.6\% (underconfident by 4.9\%), \texttt{\{ratio\}}=2.4x, and the per-repository line covered 21 held-out \texttt{sympy/sympy} tasks (pass rate 85.7\%, mean stated probability 68.0\%).

This card was then followed by the same direct calibration prompt shown in Section~\ref{sec:benchmark}. The live-scaffold bidding follow-up used the same held-out history idea, but converted it into the stricter hard-prior bid rule described in Section~\ref{sec:scaffold}.

\begin{table}[t]
\caption{Aggregate six-model calibration summary with the self-knowledge card}
\label{t:self-knowledge-summary}
\begin{tabular*}{\textwidth}{@{\extracolsep\fill}lcccc}
\toprule
Prompt condition & Mean $p$ & Pass rate & Skill & Est./actual toks \\
\midrule
Direct calibration prompt & 72.7\% & 79.0\% & $-0.107$ & 0.1929 \\
Self-knowledge card + direct prompt & 80.6\% & 79.0\% & $-0.022$ & 0.2501 \\
\bottomrule
\end{tabular*}
\note{Both rows aggregate the same six-model, 93-task setup rather than reporting a single model. ``Skill'' is aggregate Brier skill relative to a base-rate benchmark. ``Est./actual toks'' is the median estimated-to-actual token ratio, so larger values indicate less severe underestimation. The corresponding ECE values, reported in the text, improve from 0.1065 to 0.0616.}
\end{table}

\section{Proof of Proposition~\ref{prop:market-dominance}}\label{app:proof-market-dominance}

\begin{proof}
Under the three alternative mechanisms, expected welfare is
\[
W_H(d)=\lambda_{10}(-c_H)+\lambda_{01}(v-c_H)+\lambda_{11}(v-c_H)+\lambda_{00}(-c_H),
\]
\[
W_L(d)=\lambda_{10}(v-c_L)+\lambda_{01}(-c_L)+\lambda_{11}(v-c_L)+\lambda_{00}(-c_L),
\]
and
\[
W_P(d)=\left(\lambda_{10}+\lambda_{01}+\lambda_{11}\right)v-c_H-c_L.
\]
Subtracting from market welfare yields
\[
W_M(d)-W_H(d)=\lambda_{10}(v+c_H-c_L)+\lambda_{11}(c_H-c_L)+\lambda_{00}c_H,
\]
\[
W_M(d)-W_L(d)=\lambda_{01}(v+c_L-c_H)+\lambda_{00}c_L,
\]
and
\[
W_M(d)-W_P(d)=\lambda_{10}c_H+\lambda_{01}c_L+\lambda_{11}c_H+\lambda_{00}(c_H+c_L).
\]
Each expression is weakly positive when $v>c_H>c_L>0$, and it is strictly positive whenever the corresponding alternative makes a mistake with positive probability.
\end{proof}

\end{document}